\documentclass[10pt,conference,letterpaper]{IEEEtran}
\usepackage{times}
\usepackage{amssymb,bbm}
\usepackage{todonotes}
\usepackage{stfloats}
\usepackage[linesnumbered,ruled,vlined]{algorithm2e}
\usepackage{amsthm}
\usepackage[english]{babel}
\usepackage{amsmath,textcomp,enumerate,bbm,latexsym}
\usepackage{cite}
\usepackage{hyperref}
\usepackage{algpseudocode}
\usepackage{graphics}
\usepackage{epstopdf}
\usepackage{aliascnt}
\usepackage{csquotes}

\hypersetup{pdfborder= {0 0 1}}

\newcommand{\nosymbol}{}
\newcommand{\tmem}[1]{{\em #1\/}}
\newcommand{\tmop}[1]{\ensuremath{\operatorname{#1}}}
\newcommand{\tmmathbf}[1]{\ensuremath{\boldsymbol{#1}}}
\newcommand{\noplus}{}
\newenvironment{enumeratenumeric}{\begin{enumerate}[1.] }{\end{enumerate}}
\newtheorem{lemma}{Lemma}
\newtheorem{proposition}{Proposition}
\newtheorem{theorem}{Theorem}
\begin{document}


\title{Similar Handwritten Chinese Character Discrimination by Weakly Supervised Learning}

\author{\IEEEauthorblockN{Zhibo Yang\IEEEauthorrefmark{1}\IEEEauthorrefmark{2},
Huanle Xu\IEEEauthorrefmark{2}, Keda Fu\IEEEauthorrefmark{1}, Yong Xia\IEEEauthorrefmark{1}\IEEEauthorrefmark{3}}
\IEEEauthorblockA{
\IEEEauthorrefmark{1}School of Computer Science,
Harbin Institute of Technology\\
\IEEEauthorrefmark{2}Department of Information Engineering,
The Chinese University of Hong Kong\\
\IEEEauthorrefmark{2}\{yz014,xh112\}@ie.cuhk.edu.hk, \IEEEauthorrefmark{3}xiayong@hit.edu.cn}}

\maketitle
\begin{abstract}
Traditional approaches for handwritten Chinese character recognition suffer in classifying similar characters. In this paper, we propose to discriminate similar handwritten Chinese characters by using weakly supervised learning. Our approach learns a \textit{discriminative SVM} for each similar pair which simultaneously localizes the discriminative region of similar character and makes the classification. For the first time, similar handwritten Chinese character recognition (SHCCR) is formulated as an optimization problem extended from SVM. We also propose a novel feature descriptor, Gradient Context, and apply bag-of-words model to represent regions with different scales. In our method, we do not need to select a sized-fixed sub-window to differentiate similar characters. This \enquote{unconstrained} property makes our method well adapted to high variance in the size and position of discriminative regions in similar handwritten Chinese characters. We evaluate our proposed approach over the CASIA Chinese character data set and the results show that our method outperforms the state of the art.
\end{abstract}


\section{Introduction}
As the demand of optical character recognition (OCR) applications has increased tremendously in recent years, the researches on OCR have paid great attention, especially to the unconstrained handwritten Chinese character recognition (HCCR). Many progresses have been achieved in HCCR in recent decades \cite{dai2007chinese,gao2008high,kimura1987modified,liu2007normalization,liu2005pseudo}. However, the experimental results on certain databases still cannot satisfy the requirement of real application or human’s recognition ability in regard of the accuracy and efficiency. Chinese character recognition is a large-scale classification problem which involves more than 5000 frequently-used characters and no less than 10000 for the whole character set, and there are thousands of pairs of similar Chinese characters. Usually, most HCCR systems adopt “one fits all” model, but in fact, the state-of-the-art classifiers used are easily confused at classifying these similar pairs, which pulls down the overall recognition accuracy. Therefore, solving the problem of similar Chinese character recognition will bring potential improvements to the accuracy of handwritten Chinese character recognition.

Similar Chinese characters often share the same radical or differ from each other in some subtle part, such as a stroke or a dot, which we denote as \textit{discriminative region} (DR) in this paper. See Figure \ref{sim_chars} for illustration: \enquote{板} and \enquote{扳} are different in their left radical; \enquote{拨} and \enquote{拔} are different in the upper-left stroke. These similar characters are hard to classify by using existing approaches because they usually employ a global statistical model which may neglect the local detailed features that are discriminative for differentiating those similar pairs. There are also some literatures proposed to address the aforementioned problem. T. F. Gao et al. \cite{gao2007lda}\cite{gao2008high} calculate a complementary distance on a discriminative vector whose discriminability is evaluated by Linear Discriminative Analysis (LDA) and compound it with the output of a baseline classifier to help differentiate similar characters. K.C. Leung and C.H. Leung \cite{leung2010recognition} proposed to use “critical region analysis” technique, which highlights the critical regions by Fisher’s discriminant, to discriminating one character from another similar character. B. Xu et al. \cite{xu2010similar} also proposed to detect the most critical region for each pair of similar characters by using Average Symmetric Uncertainty (ASU). In \cite{shao2011multiple}, the similar Chinese characters recognition is defined as a Multiple-Instance learning problem, Shao et al. employ the Adaboost framework to select the discriminative region for each pair of similar characters. Recently, D. Tao et al. \cite{tao2014similar} introduced the discriminative locality alignment (DLA) approach to discriminate each character from other similar characters.
\begin{figure}
\centering
\includegraphics[width=.43\textwidth]{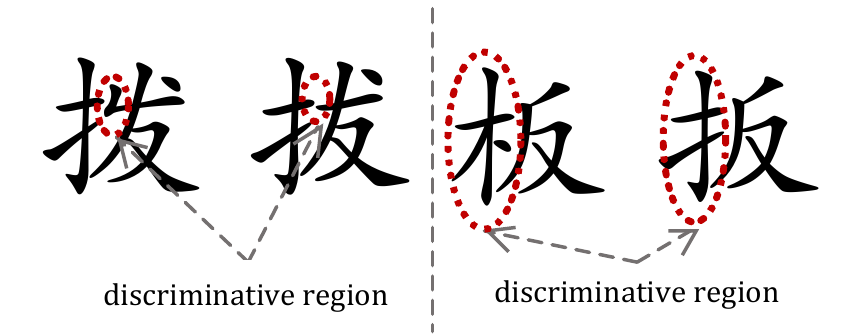}
\caption{Discriminative regions for \enquote{拨} and \enquote{拔}, \enquote{板} and \enquote{扳}}
\label{sim_chars}
\vspace{-2.1 em}
\end{figure}

The methods proposed in \cite{xu2010similar,leung2010recognition} select a sub-window with fixed size and position for each pair of similar characters which is not applicable because the critical region for each pair of handwritten Chinese characters may shift both in position and scale due to the writing style. Shao et al. \cite{shao2011multiple} alleviated this problem by training a weak classifier for each similar pair which adaptively selects the critical region with several predefined window scales and locations of each testing sample \cite{shao2011multiple}. However, the problem is not fully solved; there is still room for further improvement. Moreover, experimental results in \cite{park2000ocr,vamvakas2010handwritten} showed that approaches that employ a hierarchical treatment of patterns have considerable advantages over “one-model-fits-all” approaches, not only improving recognition accuracy, but also reducing the computational cost. So, in this paper, we propose to use two-stage classification strategy where we use MQDF \cite{kimura1987modified}, a successful classifier frequently used in HCCR, as the baseline classifier and learn a SVM as the second-level classifier that jointly localizes the discriminative region of each pair of similar characters and makes the classification. The methods proposed in \cite{xu2010similar,leung2010recognition,shao2011multiple} have to fix the window size because otherwise the feature extracted will be in different dimensionality which is unacceptable for training the classifier. While in our proposed approach, we also remove this constraint on the scale of critical region of each similar pair by introducing a novel SIFT-alike feature to the discrimination of similar Chinese characters.

In the rest of this paper, an overview of the proposed method is given in section 2. Section 3 introduces the proposed SIFT-alike feature descriptor. Section 4 presents the detailed learning algorithm for SHCCR. Finally, Section 5 gives the experimental results and the analysis of the results, and we will come to the conclusion of this paper in section 6.

\section{System Overview}
In this section,we give a brief overview of our proposed HCCR System equipped with SHCCR, see Figure \ref{sys_oview} for the flow chart. The system mainly consists of two components, the MQDF for traditional HCCR and the discriminative SVM classifier for SHCCR. The traditional methods for HCCR are well developed and widely used in many real systems, we refer readers to \cite{dai2007chinese,kimura1987modified,liu2007normalization,liu2008handwritten,liu2005pseudo} for the detailed algorithms. In this paper, we focus on SHCCR.

For the integration of the baseline classifier (MQDF) and the SVM classifier for SHCCR, we propose to
use \textit{logistic regression} $h_\theta(x)$ to do confidence evaluation \cite{liu2005classifier} over the outputs of baseline classifier corresponding to the scores of the top two candidates. Given a testing sample character $d$, MQDF is used to output the scores $(s_1,s_2)$ of the top two candidates $(c_1,c_2)$, where $c_1$ yeilds higher probability to be the true class, i.e., $s_1>s_2$. $(s_1,s_2)$ is then fed to the trained logistic function to output the confidence $h􏰯_\theta(s_1,s_2)$. If the confidence $h􏰯_\theta(s_1,s_2)$ is below an acceptable confidence􏰚 $\sigma$, we will check whether the top two candidates $(c_1,c_2)$ is in our similar characters set. Discriminative SVM will be applied to determine which class $d$ belongs to if a match is found; otherwise, $d$ will be classified as $c_1$ according to the outputs of the baseline classifier.
\begin{figure}
\centering
\includegraphics[width=.45\textwidth]{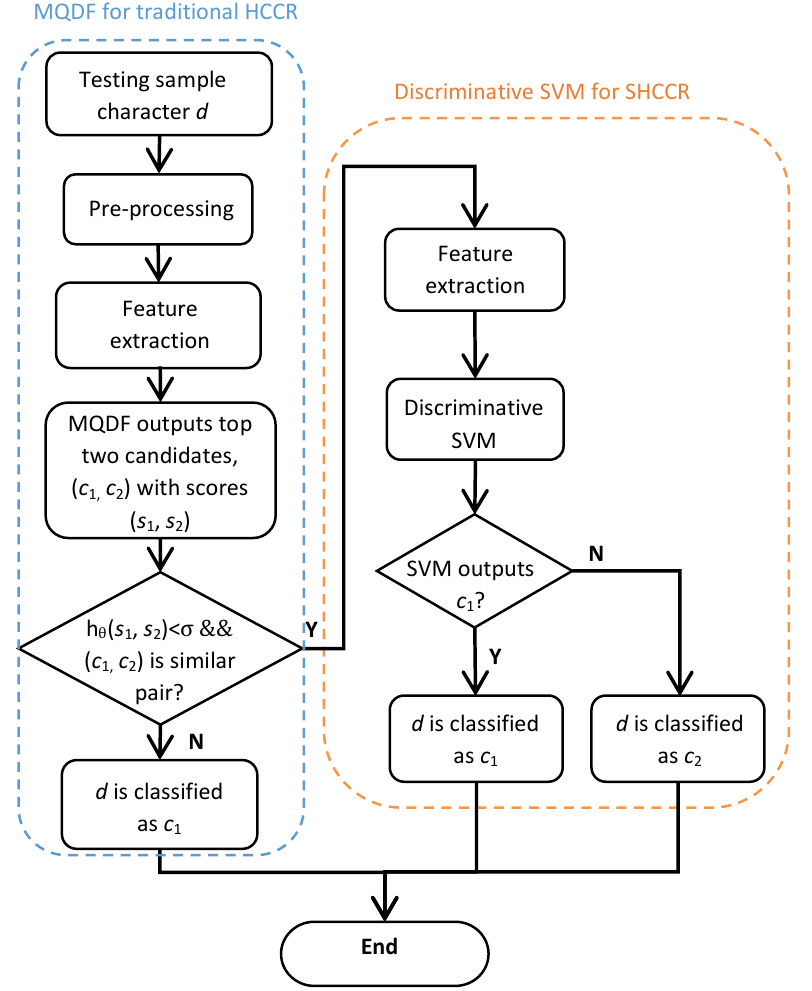}
\caption{System overview}
\label{sys_oview}
\vspace{-2.1 em}
\end{figure}

\section{Feature Construction for SHCCR}
Scale-Invariant Feature Transform (SIFT) \cite{lowe2004distinctive} is a very popular feature extraction strategy widely adopted in many computer vision tasks. However, the original SIFT cannot be directly applied to HCCR, because there are many different writing style in handwritten Chinese characters, which is more complex than variances in rotation and scale. There are also proposed works \cite{rodriguez2008local,zhang2009character} that attempt to modify SIFT and make it available for handwritten character recognition, but they are designed for the global classifier in HCCR. In this section, we present our proposed SIFI-alike feature, \textit{Gradient Context} (GC), specially designed for SHCCR.

Inspired by SIFT and SCIP \cite{nguyen2008symbol}, we sample our seed points from the points in the external contours of each normalized character image and extract the local feature descriptor by employing our proposed Gradient Context (a modified version of Shape Context \cite{belongie2002shape}) and the bag-of-visual-words (BoVW) representation method. Then, we obtain a visual dictionary for each similar pair via using the K-means algorithm to capture all feature descriptors \cite{leung2001representing}. 
Finally, we compute a histogram which counts the number of each codeword in a selected region with random scale. Figure \ref{feat_rep} gives a diagram of our proposed feature extraction procedure.

\subsection{Seeds Selection}
In the view of human perception, people can easily recognize characters as long as they are given the contour of the character image. Thus, a character can be represented by a set of discrete points sampled from the internal or external contours of the character. Therefore, in our approach, we sample the set of seed points $\mathbf{P}=\{p_1,p_2,\cdots,p_n\},p_i\in \mathbb{R}^2$ as locations of the pixels on the external contours of each character detected by an edge detector (e.g. Sobel operator). However, there is no need to keep all pixels on the contour, so as to save computational resources, we can obtain as good approximation to the underlying continuous contour as desired by keeping $n$ to be sufficiently large. In this paper, we select one pixel from every two consecutive pixels on the contour as keypoint. In this way, we can obtain $200\sim 400$ keypoints for each character in the similar pairs.

\begin{figure}
\centering
\includegraphics[width=.43\textwidth]{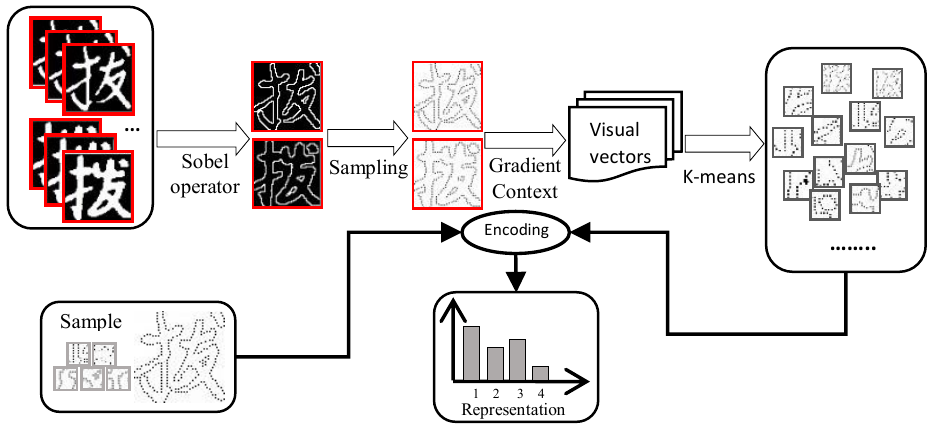}
\caption{Diagram of feature extraction procedure.}
\label{feat_rep}
\vspace{-2.1 em}
\end{figure}

\subsection{Feature Descriptor}
For each seed point $p_i$, we want to calculate a discriminative descriptor that can represent its local feature. Chinese characters are generally composed of strokes representing different directions, thus gradient features are more informative than other features in HCCR. Many experimental results \cite{liu2007normalization,liu2008handwritten} also show that gradient features perform very well in handwritten character recognition. Thus in our approach, we propose to capture the distribution of the gradients in the neighborhood of each seed point via using the log-polar histogram bins \cite{belongie2002shape}. 

For each character, the numerical computation of its gradients is implemented by employing Sobel operator on each pixel of the character image. Then, for each seed point $p_i$(the origin of the log-polar histogram bins), we compute a coarse histogram $h_i$ of the relative gradient distribution of other seed points at a certain neighborhood $\Omega$ parameterized by four radius $\{r_i\}^4_{i=1}$ of $p_i$:
\begin{equation}\label{feat_dspt}
h_i(k)=\sum G_q,\,\,\,\, q-p_i\in bin(k), q\neq p_i,q\in \Omega
\end{equation}
where $G_q$ is the gradient at point $q$. We define histogram $h_i$ as the Gradient Context of point $p_i$. For discriminating similar Chinese characters, the feature extracted should be able to represent its locality. Equation (\ref{feat_dspt}) complies with this principle by making the generated feature focused more on the points nearby than the points far away. Moreover, the size of neighborhood $\Omega$ should not be too large since otherwise the histogram $h_i$ will bring information from outside of the discriminative region which may influence the discriminability of the feature. As Shape Context \cite{belongie2002shape}, we equally divide the neighborhood into 8 panels with the same degree size with regard to their directions, and $4$ spiral bins centering at $p_i$ are used to separate its neighborhood into $4$ pieces with different sizes. As a result, there are $32$ histogram bins in total, which renders the visual vectors to be 32-dimension.

\subsection{Visual Dictionary Learning}
For each pair of similar handwritten Chinese characters, given the collection of the extracted visual vectors from all training samples, we learn a visual dictionary by employing K-means algorithm to cluster all the visual vectors. Experientially, the clusters with a too small number of members should be further pruned out. Each cluster center is defined as a codeword of the learned dictionary. Then each 32-dimension visual vector can be represented as the index of the closest codeword in the learned dictionary, which reduces the computational cost substantially. We consider histogram as a robust and compact, and yet highly discriminative descriptor. Therefore, we use a coarse histogram, which counts the number of the appearance of each codeword in the visual dictionary, to represent the local features in a sub-window with arbitrary size.

\section{Discriminative SVM for SHCCR}
Given a pair of similar character sets, $\mathbf{A}$ and $\mathbf{B}$. In general, $\mathbf{A}$ and $\mathbf{B}$ can be different with each other in two ways. One is the case that $\mathbf{A}$ has some parts (strokes) in presence while character $\mathbf{B}$ has not (e.g. Chinese characters \enquote{玉} and \enquote{王}, \enquote{本} and \enquote{木}), and vice versa. The other case is that there are different radicals or strokes in the same region of both character $\mathbf{A}$ and $\mathbf{B}$. For example, similar Chinese characters \enquote{海} and \enquote{悔} are different in there left radical but share an identical right radical \enquote{每}; \enquote{目} and \enquote{日} differ from each other in their central parts. However, in our approach, we define the second category as a special case of the first one since the fact that \enquote{海} and \enquote{悔} are different in their left radical can also be perceived as \enquote{海} has the radical \enquote{氵} while \enquote{悔} has not. Therefore, for all the similar pairs of Chinese characters, our task is to localize the most discriminative parts which appear in only one of the two similar Chinese characters and distinguish them. Most existing methods \cite{xu2010similar,gao2007lda,leung2010recognition,shao2011multiple} tackle this task within two independent steps: localization of the discriminative region and making decision. However, localization separated from the recognition will cause information loss. So in this paper, we propose to learn a SVM that jointly performs DR localization and classification.

For each pair of similar Chinese character sets, we define the one of them as positive and the other negative. There is no particular rule in deciding which character should be positive or negative, in our approach, you can either define character $\mathbf{A}$ as positive or $\mathbf{B}$ as positive. The two similar characters have equal priority during training. Then our goal is to find a region that exists in the positive class while not in the negative class. This is similar to learning to detect objects given training examples with weakly labeled (binary) data indicating the presence of an object (not location). As our problem is transformed into an object detection problem, many powerful tools can be applied, such as SVM, multiple-instance learning and weakly supervised learning. Inspired by \cite{nguyen2009weakly,hoai2014learning}, we firstly formulate the SHCCR as an optimization problem extened from SVM and propose a subgradient algorithm to solve the problem.

\subsection{SVM Learning}
Given a set of training samples of a pair of similar Chinese characters, let $D^+=\{d^+_1,d^+_2,\cdots,d^+_n$\} and $D^-=\{d^-_1,d^-_2,\cdots,d^-_n\}$ denote the positive set and negative set, respectively, and each sample $d$ is represented by a coarse histogram. Our goal of finding the most discriminative region is equal to learning a SVM with the maximum margin between two classes of data, where the data is represented by the feature extracted from all sub-windows of each training sample from $D^+\cup D^-$. Let 􏴁$\Psi(d_i)$ denotes the set of feature vectors extracted from all possible sub-windows of training sample $d_i$. Then we formulate SHCCR as the following optimization problem:
\begin{eqnarray}
  \underset{\omega, b, \xi}{\tmop{minimize}} \text{ \ } & \frac{1}{2} \| \omega \|^2 + C {\sum_{i=1}^{n+m} \xi_i} \label{opt_obj}\\
  s.t. \text{ \ } & \underset{x \in \Psi (d^+_i)}{\max} \{ \omega^T x + b \}
  \geq 1 - \xi_i \quad 1 \leq i \leq n \label{opt_c1} \\
  & \underset{x \in \Psi (d^-_j)}{\max} \{ \omega^T x + b \} \leq - 1 + \xi_{j+n}
  \   1 \leq j \leq m \label{opt_c2} \\
  & \xi_i \geq 0 \quad  1 \leq i \leq n+m \nonumber
\end{eqnarray}
where $\omega$ is a normal vector with $b$ being the bias, $\xi_i$ is the slack variable that allows some violations in the data during the training and $C$ denotes the trade-off coefficient. The constraints in the above learning objective suggest that there must be a positive sub-window in each sample from the positive set $D􏳾^+$, while all sub-windows in $D􏰊^-$ are supposed to be classified as negative which is similar to the support vector machine for multiple-instance learning \cite{andrews2002support}. By solving the above optimization problem, we obtain a soft-margin SVM classifier which can simultaneously detect the most discriminative region, and at the same time, make the decision.

Given a testing sample $d$, we first localize the discriminative region by finding the feature vector (representing a certain sub-window of $d$) that yields the maximum SVM score:
$$\hat{x}=\arg \underset{x\in \Psi(d)}{\max}(\omega^T x+b)$$ 
Then, the classification is made corresponding to the value of
$\hat{y}=\omega^T \hat{x}+b$.
If $\hat{y}>0$, the testing sample $d$ will be classified as positive; otherwise negative.

\subsection{Solution approach to SVM learning}
In this section, we propose to employ the optimization skills adopted from the Cutting Plane Algorithm as well as the standard convex optimization problem to solve our learning objective (\ref{opt_obj}) whose constrains are generally non-convex.

For the optimization objective (\ref{opt_obj}), the non-convex constrains (\ref{opt_c1}) is difficult to handle. To tackle this issue, we first rewrite the original formulation as the following unconstrained optimization problem:  

\begin{equation}
\label{unconstrained_optimization}
 \underset{{\mu}}{\min} \left\{ f ({\mu}) + C
   \overset{n}{\underset{i = 1}{\sum}} g_i^{} ({\mu}) + C
   \overset{m}{\underset{j = 1}{\sum}} h^{}_j ({\mu}) \right\} 
 \end{equation}

where $\mu = (\omega, b)$ and 
\begin{eqnarray*}
& f ({\mu}) = \frac{1}{2} \| \omega \|^2 & \\
  & g_i^{} ({\mu}) = \max \left\{ - 1, \underset{x \in \Psi
  (d^+_i)}{\min} \phi_{}^{} (x, {\mu}) \right\}; \ \phi_{}^{} (x, {\mu}) = - \omega^T x - b & \\
  & h^{}_j ({\mu})^{} = \max \left\{ - 1, \underset{y \in \Psi
  (d^-_j)}{\max} \varphi (y, {\mu}) \right\}; \ \varphi (y, {\mu}) =
  \omega^T y + b & 
\end{eqnarray*}


Our strategy is to convert the complex learning objective (\ref{opt_obj}) to the traditional SVM optimization problem which can be solved using the well-developed tools. This can be achieved with the help from the following lemma. 

\begin{proposition}
  Let $\{ f_i \}_{i \in I}$ be an arbitrary family of convex functions on
  {\tmem{}}$\mathbbm{R}^n$. Then, the pointwise supremum $f = \sup_{i \in I}
  f_i$ is convex. 
  \label{proposition_1}
\end{proposition}

\begin{proof}
  Refer to \cite{convex_optimization} for the detailed proof. 
\end{proof}

\begin{lemma}
\label{lemma_1}
  The following function
  \begin{equation}
  \label{function_of_mu}
   \Phi ({\mu}) = f(\mu) + C\overset{m}{\underset{j =
     1}{\sum}} h^{}_j ({\mu})
  \end{equation}
  is a convex function of ${\mu} \nosymbol$. 
\end{lemma}

\begin{proof}
  Firstly, observe that $\varphi (y, {\mu}) =
  \omega^T y + b$ is a linear function, applying Proposition \ref{proposition_1} here, it's straightforward to show that $h^{}_j ({\mu})^{} = \max \bigg\{ - 1, \underset{y \in \Psi(d^-_j)}{\max} \varphi (y, {\mu}) \bigg\}$ is convex. Secondly,  $f(\mu)$ is a convex function of $\mu$. This completes the proof. 
\end{proof}

Lemma \ref{lemma_1} indicates that if we fix the sub-window for each positive sample in constrains (\ref{opt_c1}), the formulation in \eqref{unconstrained_optimization} becomes a convex optimization problem. Therefore, we first use the Cutting Plane Algorithm (also used in multiclass maximum margin clustering by Zhao et al. (Zhao et al., 2008)) to select the most violated feature vector extracted from the sub-windows of each positive sample at every iteration. See Algorithm \ref{algorithm_outer} for the detailed implementation.
\vspace{-0.5em}
\begin{algorithm}
 \label{algorithm_outer}
	 \KwData{$D^+$ and $D^-$}
	 \KwResult{$\omega$ and $b$}
	 Initialize $\Omega=\varnothing$\;
	 \Repeat {$tv<\tau$}{
	 	$tv=0$ \;
	 	\For{$1\le i \le n$}{
	 		Select the most violated feature:
	 		\begin{equation}
	 			\hat{x}_i =\arg \underset{x\in \Psi(d_i^+)}{\max}(\omega^T x+b)
	 		\end{equation}
	 		$\Omega = \Omega \cup \{\hat{x}_i\}$ \\
	 	}
	 	\label{K-means-type}
	 	Optimize for SVM with constraints (\ref{opt_c2}):
	 	\begin{eqnarray} \label{opt_ccp}
		  &(\omega,b,\xi)=\arg \min\big(\frac{1}{2} \| \omega \|^2 + C {\sum_{i=1}^{n+m} \xi_i}\big)\\
		  s.t. 
		  &  \omega^T x + b  \geq 1 - \xi_i
		  \quad x\in \Omega, 1 \leq i \leq n \nonumber \\
		  & \underset{x \in \Psi (d^-_j)}{\max} \{ \omega^T x + b \}
		  \leq -1 + \xi_{j+n}
		  \quad 1 \leq j \leq m \nonumber\\
		  & \xi_i \geq 0 
		  \quad 1 \leq i \leq n+m \nonumber
		\end{eqnarray}
		\label{optimization_algorithm_convex}
	 }
 \caption{Cutting Plane Algorithm for Constrains (\ref{opt_c1})}
\end{algorithm} 


As is illustrated in Algorithm \ref{algorithm_outer}, we make the large constraint set (\ref{opt_c1}) to be size-manageable by adding the most violated constraint for each positive training sample. This algorithm stops when the total violation at one iteration is smaller than an acceptable precision􏴑. The following theorem guarantees that Algorithm \ref{algorithm_outer} converges to a local optimum of our learning objective \eqref{opt_obj}.

\newcounter{mytempeqncnt}
\begin{figure*}[ht]
\normalsize
\setcounter{mytempeqncnt}{\value{equation}}
\setcounter{equation}{9}
\begin{eqnarray}
  \label{subgradient_1}
   \frac{\partial \Gamma ({\mu})}{\partial {\mu}} & = \text{}
   \left(\begin{array}{c}
     \omega^{\ast}\\
     0
   \end{array}\right) + C \text{} \sum_{i = 1}^n \left(\begin{array}{c}
     - \hat{x}_i\\
     - 1
   \end{array}\right) \tmmathbf{1}_{\phi (\hat{x}_i, {\mu}) \geqslant -
   1} + C\sum_{j = 1}^m \left(\begin{array}{c}
     y^{_{} \ast}_j\\
     1
   \end{array}\right) \tmmathbf{1}_{\varphi (y^{\ast}_j, {\mu}) \geqslant
   - 1}\\
   \label{optimal_y}
   & y_{j}^{\ast} = \arg
  \underset{y \in \Psi (d^-_j)}{\max} \{\omega^T y + b\}
  \end{eqnarray}
\setcounter{equation}{\value{mytempeqncnt}}
\hrulefill
\vspace{-.5em}
\end{figure*}

\begin{theorem}
\label{convergence_theorem}
Algorithm \ref{algorithm_outer} stops within finite number of iterations and converges to a local optimum of the learning objective \eqref{opt_obj}.
\end{theorem}

Refer to Appendix \ref{proof-theorem-1} for the detailed proof. We proceed to solve the convex optimization problem \eqref{opt_ccp} with constrains (\ref{opt_c2}) under the current constraint set 􏰉at each iteration. Similarly, \eqref{opt_ccp} is equivalent to minimizing the following unconstrained optimization problem over $\mu$:

\begin{equation}
  \label{function_of_mu}
   \Gamma ({\mu}) = f(\mu) + C\overset{n}{\underset{i = 1}{\sum}} \max \{ - 1,
     \phi_{}^{} (\hat{x}_i, {\mu}) \} + C\overset{m}{\underset{j =
     1}{\sum}} h^{}_j ({\mu})
 \end{equation}
Based on Lemma \ref{lemma_1}, $\Gamma ({\mu})$ is a convex function of $\mu$. However, $\Gamma(\mu)$ is not differentiable in general and thus the traditional gradient descent approach is not applicable to this problem. Hence, we adopt the subgradient method to handle it. To begin with, we derive the subgradient of $\Gamma(\mu)$ via the following lemma: 

\begin{lemma}
\label{lemma_2}
  The subgradient of \ $\Gamma(\mu)$ is determined by Equation \eqref{subgradient_1} and Equation \eqref{optimal_y} where $\tmmathbf{1}_{A}$ is equal to 1 if condition A holds and 0 otherwise.
\end{lemma}

Refer to Appendix \ref{proof-lemma-2} for the detailed proof. Following Lemma \ref{lemma_2}, we design the subgradient descent algorithm below to solve the optimization problem \eqref{opt_ccp}, where $\alpha_k$ is the stepsize. We design the step size in the theorem below to guarantee the convergence of Algorithm \ref{algorithm_2}. 

\vspace{-0.5em}
\begin{algorithm}
\label{algorithm_2}
\begin{enumeratenumeric}
  \item Initialize $k=1$ and $\mu_k = \mu$;

  \item Repeat the following steps until convergence;
  
  \item In the $k$th step, Fix ${\mu_k}$, optimize over $\phi$ and \ ${y_{j,k}^{\ast}} = \arg
  \underset{y\in \Psi (d^-_j)}{\min} \phi_{}^{} (y, {\mu_k})$, for $j = 1,
  \ldots, m$;
  
  \item Update $\mu_{k+1}$ based on the following equation:
  $$
  \mu_{k+1} = \mu_{k} + \alpha_k \frac{\partial \Gamma ({\mu})}{\partial {\mu}}\Big|_{\mu = \mu_{k},y_j={y_{j,k}^{\ast}}}
  $$
\end{enumeratenumeric}

\caption{Subgradient Descent Algorithm}
\end{algorithm}
\vspace{-1em}

\begin{theorem}
\label{theorem_stepsize}
By choosing $\alpha_k = \frac{1}{k}$, the sequences generated by Algorithm \ref{algorithm_2} converge to the global optimal solution to the optimization problem \eqref{opt_ccp}. 
\end{theorem}

Refer to Appendix \ref{proof_convergence} for the detailed proof of Theorem \ref{theorem_stepsize}. Even though our optimization approach converges, the algorithm still needs to search for the feature vector that maximizes the SVM score from all possible sub-windows of each training sample iteratively. Therefore, the algorithm requires a very fast optimal region localization procedure. In our approach, we propose two strategies to expedite the localization process: one is employing an efficient sub-window searching algorithm proposed by Lampert et al. \cite{lampert2009efficient}; the other is by predefining a sliding step and several scale templates corresponding to the structures of Chinese characters. According to our experimental results, the second strategy is much more efficient and achieves nearly as good performance compared to the first one. Moreover, \textit{Integral Image} can also be used to speed up the search.

\begin{table*}[t] 
	\begin{center} 
	\caption{Recognition accuracy of the proposed method with different $\sigma$}
	\begin{tabular}{p{2cm}|*{10}{p{1cm}}}
	 \hline
	 & 0.7 & 0.8 & 0.9 & 0.92 & 0.94 & 0.95 & 0.96 & 0.97 & 0.98 & 1.00\\ 
	 \hline
	 Accuracy (\%) & 97.90 & 97.95 & 98.13 & 98.22 & 98.28 & 98.28 & 98.29 & 98.25 & 98.19 & 98.05\\ 
	 \hline
	\end{tabular}
	\end{center}
	\label{table_para}
	\vspace{-1.8em}
\end{table*}

\section{Experiments}

To evaluate the effectiveness of our proposed algorithms for SHCCR, we first compare the our approach with the traditional MQDF. Then, we conduct another group of experiments to compare the discriminative SVM for SHCCR with two baselines, which are the critical region selection by Average Symmetric Uncertainty (ASU) \cite{xu2010similar} and by Multiple-Instance Learning (MIL) \cite{shao2011multiple} respectively. The comparison are made based on the same set of similar Chinese characters under our Logistic Regression (LR)-based framework. All of our experiments are conducted over a large handwritten Chinese character database, CASIA, which contains binary images of 3,755 frequently used characters (in level-1 set of GB2312-80), and there are 300 samples in each class.

\subsection{MQDF Training}
In our approach, 270 out of 300 samples in each class are randomly selected to train MQDF, the rest of them are used for testing. For each sample image, we employ the pseudo-two-dimension bi-moment normalization method \cite{liu2005pseudo} and the normalization-cooperated gradient feature (NCGF) \cite{liu2007normalization} extraction strategy to jointly extract a 512-dimension feature vector without explicitly normalizing the image to a fixed scale ($64\times 64$ pixels in our method). Linear Discriminant Analysis (LDA) \cite{fisher1936use} is used to reduce the dimensionality of the input features from 512 to 200. It is also worth noting that, before inputting the feature vector into the classifier, we use the power transformation (also known as Box-Cox transformation) strategy \cite{van1997box} to make the feature distribution closer to the Gaussian distribution.
\begin{figure}
\centering
\includegraphics[width=.2\textwidth]{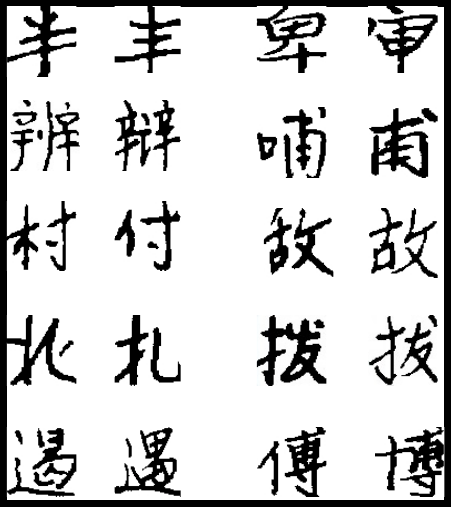}
\caption{Normalized handwritten samples of similar Chinese characters}
\label{sim_pairs}
\vspace{-.8 em}
\end{figure}

\subsection{Parameters Setup}
First of all, we need to build a set of similar Chinese characters in order to train the discriminative SVM. However, there is no standard definition of similar Chinese characters, because \enquote{similar} is an obscure concept heavily depending on human perception. In our approach, we define the similar Chinese characters set in the view of \enquote{machine perception}. The machine (baseline classifier) perceives characters that are often misclassified as similar characters. If the number of times character A is misclassified as character B and the other way around is more than a threshold $T$, A and B are deemed as a pair of similar characters. In our experiments with MQDF, 30 samples are extracted from each class (3755 classes in total) for the test, and MQDF outputs 2383 misclassified items. We obtain a set of similar Chinese character with 1105 pairs by setting $T=2$ . Figure \ref{sim_pairs} gives some examples of our obtained similar Chinese characters set.

For each pair of similar Chinese characters, we train a discriminative SVM to jointly localize the discriminative region and distinguish them. We first normalize each sample image to the scale of $64\times 64$ pixels by employing bi-moment normalization method \cite{liu2003handwritten} and then extract features from sub-windows of each sample image by using our proposed Gradient Context Operator, where we set the parameters $\{r_i\}^4_{i=1}$ of the neighborhood $\Omega$ to be $3, 4, 8$ and 16, respectively. All feature descriptors are projected to the formed visual dictionary and the locality of each sample character is represented by a coarse histogram which counts the number of occurrence of each codeword in each sub-window. Considering the efficiency of our optimization algorithm, we select 9 scales for the sliding sub-window at the space of 4 pixels according to the structures of Chinese characters. The selected scales (width, height) are as follows: {(64, 24), (24, 64), (32, 32), (16, 16), (24, 24), (16, 48), (48, 16), (64, 32), (32, 64)}. $\tau$ in Algorithm \ref{algorithm_outer} is empirically set to be 0.6.

\subsection{Evaluation of the Proposed HCCR System}
In this section, we compare our proposed approach with the traditional MQDF. Before the comparison, we need to configure some parameters in our proposed method first.

Given a well-trained MQDF, we randomly select 30 samples for each class in the CASIA database which contains 112,650 testing samples in total, and feed them to the trained MQDF which outputs the scores of the top two candidates
$(x_1,x_2)$ with an indicator y indicating whether the testing sample is correctly classified ($y=1$ means correct, $y=0$ otherwise). Therefore, we obtain a pool of data $\mathbf{X}􏰁=\{(x^i_1,x^i_2 ,y_i)\}^N_{i=1}$ where $N$ is the number of testing samples, and $\mathbf{X}􏰁$ is used to train a logistic function. In order to combine the two classifiers, namely, MQDF and discriminative SVM, we need to specify an acceptable confidence $\sigma$ for the logistic function. Table I gives the recognition accuracy of our proposed method with different $\sigma$.

\begin{table} 
	\caption{Recognition accuracy for different HCCR system}
	\vspace{-1.8em}
	\begin{center}
	\begin{tabular}{p{2cm}|p{1.5cm} p{1.5cm}} 
	 \hline
	  & MQDF & Ours \\ 
	 \hline
	 Accuracy (\%) & 97.89 & 98.29 \\ 
	 \hline
	\end{tabular}
	\end{center}
	\label{table:hccr}
	\vspace{-1em}
\end{table}

\begin{table}
	\caption{Recognition accuracy of different methods}
	\vspace{-1.8em}
	\begin{center}
	\begin{tabular}{p{2cm}|p{1.5cm} p{2cm} p{1.5cm}} 
	 \hline
	 Method& ASU & MIL & Ours \\ 
	 \hline
	 Accuracy (\%) & 98.19 & 98.00 & 98.29 \\ 
	 \hline
	\end{tabular}
	\end{center}
	\label{tab:shccr} 
	\vspace{-1.8em}
\end{table}

From Table I, we select the optimal acceptable confidence σ for our logistic regression to be 0.96. Moreover, by comparing the recognition accuracy of 􏰂$\sigma=0.96$ and 􏰂$\sigma=1.00$ (working without logistic regression), we can deduce that logistic regression will help boost the recognition accuracy in combining two different classifiers.

With our optimal acceptable confidence determined, we compare our proposed HCCR system with SHCCR component combined with the traditional MQDF. The experimental results are presented in Table \ref{table:hccr} which shows that our proposed discriminative SVM for SHCCR can significantly improve the recognition accuracy of the HCCR system.

\subsection{Evaluation of the SHCCR Component}
Most recognition systems with SHCCR component use a baseline classifier to output similar characters, and similar characters are further classified by a discriminative classifier \cite{xu2010similar,shao2011multiple}. Thus, we evaluate our proposed discriminative SVM for SHCCR by comparing it with the other two competitive methods proposed for SHCCR \cite{xu2010similar,shao2011multiple} over the same predefined set of similar Chinese characters. For comparison, all methods for SHCCR are evaluated under our Logistic Regression based framework. We denote the methods proposed in \cite{xu2010similar} and \cite{shao2011multiple} as ASU and MIL, respectively.

In ASU, the features extracted from all the critical regions are supposed to be fed into a LDA classifier for each pair of similar characters. However, based on our experiments, the calculated covariant matrix may not be positive definite due to the high-dimension feature space and limited training sample for each pair. So in our implementation of ASU, SVM, a native two-class classifier which is free of that constraint, is used to replace LDA.
For each pair of similar characters, MIL requires learning 31 weak classifiers:

\setcounter{equation}{11}
\begin{equation}
\label{adaboost}
h(I,B(x,y,s)) = \left\{ 
  \begin{array}{l}
   1, \quad d(I,B(x,y,s))\times p_w < T_w \times p_w,\\
   -1,\quad otherwise
  \end{array} 
\right.
\end{equation}
where $I$ is an instance, $B(x, y, s)$ represents a small bag, $T_w$ is a threshold and $p_w\in\{1,-1\}$. However, the optimization method for learning the parameters of the weak classifier (\ref{adaboost}) is not specified in \cite{shao2011multiple}. So it is worth noting that we use exhaustive searching to optimize parameters $I, B􏳿(x,y,s)􏴀$ and $p_w$ since they are discrete variables belonging to three finite sets. \textit{Perceptron Learning} algorithm is employed to learn an optimal value of $T_w$ by minimizing the following objective function:
\begin{equation} \label{percept_learn}
f(x)=-\sum_{i\in \Omega}distr(i)\times y_i \times T_w
\end{equation}
where $\Omega$ indexes the misclassified samples, $y_i\in \{1,-1\}$ and $distr(i)$ denote the label and the assigned weight of the $i$th training sample. \textit{Stochastic Gradient Descent} method is used to find a minimal point of $f(x)$. The detailed optimization procedure is presented in Algorithm 3. We list the experimental results in Table \ref{tab:shccr}.

\begin{algorithm} \label{algorithm_percp}
	 Initialize: $T_w=0,k=1,T=20,D=\{(x_i,y_i)|1\le i \le N\}$\\
	 \Repeat {$k\le T$}{
		\For{$1 \le i \le N$}{
	 		$\hat{y_i}=h(I,B(x,y,s))$\\
	 		$T_w=T_w+distr(i)\alpha(y_i-\hat{y_i})$
	 	}
	 	$k=k+1$\;
	 }
 	\caption{Optimization Procedure for (\ref{percept_learn})}
\end{algorithm}

\vspace{-1em}
Obviously, our method outperforms the ASU and MIL with regard to the recognition accuracy of similar Chinese character recognition. ASU first divides each image into $8\times 8=64$ small regions and then determine whether a small region is “critical” according to its Average Symmetric Uncertainty. From our experiments, we found that the critical regions selected by ASU are often dispersive and different with the regions that human may percept as “critical” given a pair of similar Chinese characters. This is probably caused by the highly variant writing styles in handwritten Chinese characters. While MIL cannot localize a certain critical region for each similar pair since each weak classifier holds a “critical region” and they are not the same with each other in each similar pair. However, our method can explicitly select a critical region for each pair of similar characters and it is showed from our experiments that the selected critical region generally complies with the human perception of “critical”. Several examples of discriminative region localization using ASU, MIL and our method are presented in Figure \ref{DR}. We list the experimental results in Table \ref{tab:shccr}.

Since our approach treats all sub-windows in the negative samples as negative, so no detection of discriminative region in negative samples should be presented here.
\begin{figure}
\centering
\includegraphics[width=.43\textwidth]{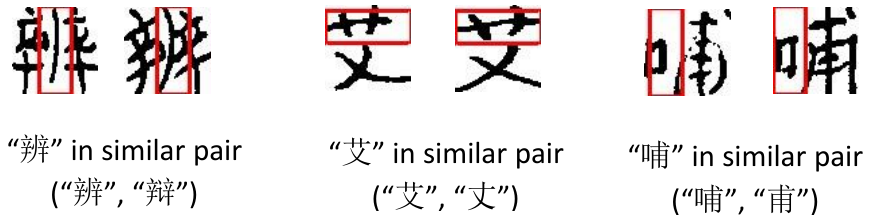}
\caption{Localization of discriminative region of pairs of similar Chinese characters}
\label{DR}
\vspace{-.8 em}
\end{figure}

\section{Conclusion}
Focusing on SHCCR, we propose to use the weakly labeled data to learn a SVM which localizes the regions with highly discriminative ability of similar characters and makes the classification simultaneously. To make our method more robust to SHCCR problem, we also propose a novel SIFT-alike feature descriptor with which we do not need to constrain the scale of the sliding window, thus improves the recognition accuracy. The experimental results show that our method is superior to several competitive approaches and improves the accuracy of handwritten character recognition. However, our method can still be improved from several aspects. First of all, during training the discriminative SVM, the fact that the discriminative regions of two similar characters generally lie in the same part of the two characters can be utilized to accelerate the optimization process. Moreover, the \enquote{unconstrained} property (no constraints are put on the scale of the sliding window) of our SIFT-alike feature is not fully exploited in this paper, it can be used to further improve the performance of our method. Therefore, our future work will focus on these aforementioned points.

\section{Acknowledgment}
This work was financially supported by National Natural Science Foundation of China (61173086). The authors would like to thank Zhenming Zhang for helping with implementing ASU algorithm, as well as the reviewers and editors for their advices.

\bibliographystyle{abbrv}
\bibliography{paper}

\appendix

\subsection{Proof of Theorem \ref{convergence_theorem}}
\label{proof-theorem-1}
\begin{proof}
\label{proof_convergence}
At the $k$-th iteration ($k > 1$) of Algorithm \ref{algorithm_outer}, it holds that $g_i (x_k^{\ast}, {\mu}^{\ast}_{k - 1}) \leq g_i (x_{k - 1}^{\ast}, {\mu}^{\ast}_{k -
1})$, where $x_k^{\ast} = \arg \underset{x \in \Psi (d^+_i)}{\min} \{ -
(\omega^{\ast}_{k - 1})^T x - b_{k - 1} \}$. Let $obj(k)$ denote
the objective value at the $k$-th iteration. Therefore, we have
\begin{eqnarray}
   {obj} (k) &  = & f ({\mu}^{\ast}_k) + C \overset{n}{\underset{i = 1}{\sum}} g_i^{} 
  (x^{\ast}_k, {\mu}^{\ast}_k) + C \overset{m}{\underset{i = 1}{\sum}}
  h^{}_i (y_k^{\ast}, {\mu}_k^{\ast}) \nonumber  \\
  \label{first_optimization}
  & \leq & f ({\mu}_{k - 1}^{\ast}) + C \overset{n}{\underset{i = 1}{\sum}}
  g_i^{} (x_k^{\ast}, {\mu}^{\ast}_{k - 1}) \nonumber \\ & &  + C \overset{m}{\underset{i
  = 1}{\sum}} h^{}_i (y_{k - 1}^{\ast}, {\mu}_{k - 1}^{\ast})  \\
  \label{second_optimization}
  & \leq & f ({\mu}_{k - 1}^{\ast}) + C \overset{n}{\underset{i = 1}{\sum}}
  g_i^{} (x_{k - 1}^{\ast}, {\mu}^{\ast}_{k - 1}) \nonumber \\ & & + C
  \overset{m}{\underset{i = 1}{\sum}} h^{}_i (y_{k - 1}^{\ast}, {\mu}_{k -
  1}^{\ast})  \\ 
   & = & {obj} (k - 1)  
\end{eqnarray}

Inequality \eqref{first_optimization} is due to Step \ref{optimization_algorithm_convex} and Inequality \eqref{second_optimization} is because of Step \ref{K-means-type}. So this means that after each iteration, the objective of $P 1$ along with the sequences generated by the algorithm above is non-increasing. In addition, it's straightforward to show that $\tmop{obj} (k)$ is bounded below by $0$. Therefore, we conclude that our algorithm converges to the local optimum.
\end{proof}

\subsection{Proof of Lemma \ref{lemma_2}}
\label{proof-lemma-2}
\begin{proof}
To prove Lemma \ref{lemma_2}, we only need to deal with the last term in Equation \eqref{subgradient_1}. Fix a $\mu$, if $\underset{y \in \Psi (d^-_j)}{\max} \{\omega^T y + b\} \leq -1$, $\frac{\partial{h_i(\mu)}}{\partial{\mu}} = 0$. In contrast, if $\underset{y \in \Psi (d^-_j)}{\max} \{\omega^T y + b\} > -1$, $h_j(\mu) = \underset{y \in \Psi (d^-_j)}{\max} \{\omega^T y + b\}$. That is, 
for each $\mu$ and $y_j^{\ast} = \arg \underset{y \in \Psi (d^-_j)}{\max} \{\omega^T y + b\}$, we have for every $\mu^{'}$:
\begin{eqnarray*}
  h_j^{} ({\mu}^{'}) & \geqslant & (w^{'})^T y^{\ast} + b^{^{'}}\\
  & = & w^{T} y + b  \noplus + (w^{'} - w)^T y^{\ast} + (b^{'} - b)\\
  & = & h_j ({\mu}) + (w^{'} - w^{})^T y^{\ast} + (b^{'} - b)
\end{eqnarray*}
It immediately follows that one subgradient of $h_j^{} ({\mu})$ is $\frac{\partial h_i ({\mu})}{\partial {\mu}} = \left(\begin{array}{c}
     y^{_{} \ast}_j\\
     1
   \end{array}\right)$.
This completes the proof. 
\end{proof}

\subsection{Proof of Theorem \ref{theorem_stepsize}}
\begin{proof}
It follows that $\sum_{k}\alpha_k = \infty$ and $\sum_{k}\alpha_k^2 < 2$. Based on Proposition 8.2.6 of \cite{convex-analysis}, the sequences $\{\mu_k\}$ generated by Algorithm \ref{algorithm_2} converge to some optimal solutions for solving Equation \eqref{function_of_mu}. Then the sequence $\{\mu_k,y_{j,k}\}$ determined by Step (3) and (4) in Algorithm \ref{algorithm_2} must also converge to one global optimal solution. This completes the proof.
\end{proof}

\end{document}